\documentclass[9pt,legal,twocolumn]{scrartcl}
\usepackage{amssymb,amsfonts,amsmath,latexsym,dsfont}
\usepackage{fullpage}
\usepackage{graphicx}
\usepackage{mathrsfs}
 \usepackage[margin=16mm]{geometry}

\graphicspath{{./figs/}}

\newcommand{\hf}{{\frac 12}}
\newcommand{\bfA}{{\bf A}}
\newcommand{\bfB}{{\bf B}}
\newcommand{\bfC}{{\bf C}}
\newcommand{\bfD}{{\bf D}}

\newcommand{\bfF}{{\bf F}}

\newcommand{\bfJ}{{\bf J}}
\newcommand{\bfK}{{\bf K}}

\newcommand{\bfW}{{\bf W}}

\newcommand{\bfY}{{\bf Y}}
\newcommand{\bfZ}{{\bf Z}}

\newcommand{\bfc}{{\bf c}}

\newcommand{\bfe}{{\bf e}}

\newcommand{\bfy}{{\bf y}}
\newcommand{\bfu}{{\bf u}}

\newcommand{\calN}{{\cal N}}

\newcommand{\bfmu}{{\boldsymbol \mu}}
\newcommand{\bfbeta}{{\boldsymbol \beta}}

\newcommand{\bftheta}{{\boldsymbol \theta}}

\let\OLDthebibliography\thebibliography
\renewcommand\thebibliography[1]{
  \OLDthebibliography{#1}
  \setlength{\parskip}{0pt}
  \setlength{\itemsep}{0pt plus 0.3ex}
}

\usepackage[font=small,labelfont=bf]{caption}

\newcommand{\diag}{\mathrm{diag}\,}

\newcommand{\R}{\ensuremath{\mathds{R}}}

\newtheorem{theorem}{Theorem}
\newtheorem{proof}{Proof}

\usepackage{hyperref}
\hypersetup{
 citecolor=blue,
 urlcolor=blue,
 pdftitle={Deep Neural Networks Motivated by Partial Differential Equations},
 pdfauthor={Lars Ruthotto and Eldad Haber},
 pdfkeywords={Machine Learning, Deep Neural Networks, Dynamic Inverse problems, PDE-Constrained Optimization, Parameter Estimation, Image Classification}
 }

\newdimen\iwidth
\newdimen\iheight

\usepackage{authblk}

\author[1,3]{Lars Ruthotto}
\author[2,3]{Eldad Haber}
\affil[1]{Emory University, Department of Mathematics and Computer Science, Atlanta, GA, USA, (\url{lruthotto@emory.edu})}
\affil[2]{Department of Earth and Ocean Science, The University of British Columbia, Vancouver, BC, Canada, (\url{ehaber@eoas.ubc.ca})}
\affil[3]{Xtract Technologies Inc., Vancouver, Canada, (\url{info@xtract.tech})}

\title{Deep Neural Networks Motivated by Partial Differential Equations}

\begin{document}
\maketitle

\begin{abstract}    
   \textbf{Abstract.} Partial differential equations (PDEs) are indispensable for modeling many physical phenomena and also commonly used for solving image processing tasks. 
        In the latter area, PDE-based approaches interpret image data as discretizations of multivariate functions and the output of image processing algorithms as solutions to certain PDEs. 
        Posing image processing problems in the infinite dimensional setting provides powerful tools for their analysis and solution. 
        Over the last few decades, the reinterpretation of classical image processing problems through the PDE lens has been creating multiple celebrated approaches that benefit a vast area of tasks including image segmentation, denoising, registration, and reconstruction.

        In this paper, we establish a new PDE-interpretation of a class of deep convolutional neural networks (CNN) that are commonly used to learn from speech, image, and video data. 
        Our interpretation includes convolution residual neural networks (ResNet), which are among the most promising approaches for tasks such as image classification having improved the state-of-the-art performance in prestigious benchmark challenges. 
        Despite their recent successes, deep ResNets still face some critical challenges associated with their design, immense computational costs and memory requirements, and lack of understanding of their reasoning.
    
        Guided by well-established PDE theory, we derive three new ResNet architectures that fall into two new classes: parabolic and hyperbolic CNNs. 
        We demonstrate how PDE theory can provide new insights and algorithms for deep learning and demonstrate the competitiveness of three new CNN  architectures using numerical experiments.
\end{abstract}


%
\section{Introduction} 
\label{sec:introduction}
Over the last three decades, algorithms inspired by partial differential equations (PDE) have had a profound impact on many processing tasks that involve speech, image, and video data. 
Adapting PDE models that were traditionally used in physics to perform image processing tasks has led to ground-breaking contributions. 
An incomplete list of seminal works includes optical flow models for motion estimation~\cite{HornSchunck1981}, nonlinear diffusion models for filtering of images~\cite{PeronaMalik1990}, variational methods for image segmentation~\cite{MumfordShah1989,AmbrosioTortorelli1990,ChanVese1999}, and nonlinear edge-preserving denoising~\cite{RudinOsherFatemi1992}.

A standard step in PDE-based data processing is interpreting the involved data as discretizations of multivariate functions. 
Consequently, many operations on the data can be modeled as discretizations of PDE operators acting on the underlying functions.
This continuous data model has led to solid mathematical theories for classical data processing tasks obtained by leveraging the rich results from PDEs and variational calculus (e.g.,~\cite{scherzer2009variational}).
The continuous perspective has also enabled more abstract formulations that are independent of the actual resolution, which has been exploited to obtain efficient multiscale and multilevel algorithms (e.g.,~\cite{modersitzki2009fair}).

In this paper, we establish a new PDE-interpretation of deep learning tasks that involve speech, image, and video data as features. 
Deep learning is a form of machine learning that uses neural networks with many hidden layers~\cite{bengio2009learning,lecun2015deep}.
Although neural networks date back at least to the 1950s~\cite{Rosenblatt1958}, their popularity soared a few years ago when deep neural networks (DNNs) outperformed other machine learning methods in speech recognition~\cite{RainaEtAl2009} and image classification~\cite{hinton2012deep}.
Deep learning also led to dramatic improvements in computer vision, e.g., surpassing human performance in image recognition~\cite{hinton2012deep,KrizhevskySutskeverHinton2012,lecun2015deep}.
These results ignited the recent flare of research in the field. 
To obtain a PDE-interpretation, we use a continuous representation of the images and extend recent works by~\cite{HaberRuthotto2017,E2017}, which relate deep learning problems for general data types to ordinary differential equations (ODE).

Deep neural networks filter input features using several layers whose operations consist of element-wise nonlinearities and affine transformations. 
The main idea of convolutional neural networks (CNN)~\cite{LeCunBengio1995} is to base the affine transformations on convolution operators with compactly supported filters. 
Supervised learning aims at learning the filters and other parameters, which are also called weights, from training data. 
CNNs are widely used for solving large-scale learning tasks involving data that represent a discretization of a continuous function, e.g., voice, images, and videos~\cite{KrizhevskySutskeverHinton2012,LeCunBengio1995,LeCunKavukcuogluFarabet2010}. 
By design, each CNN layer exploits the local relation between image information, which simplifies computation~\cite{RainaEtAl2009}. 

Despite their enormous success, deep CNNs still face critical challenges including designing a CNN architecture that is effective for a practical learning task, which requires many choices. 
In addition to the number of layers, also called depth of the network, important aspects are the number of convolution filters at each layer, also called the width of the layers, and the connections between those filters. 
A recent trend is to favor deep over wide networks,  aiming at improving generalization (i.e., the performance of the CNN on new examples that were not used during the training)~\cite{lecun2015deep}. 
Another key challenge is designing the layer, i.e., choosing the combination of affine transformations and nonlinearities. 
A practical but costly approach is to consider depth, width, and other properties of the architecture as hyper-parameters and jointly infer them with the network weights~\cite{hernandez2016general}. 
Our interpretation of CNN architectures as discretized PDEs provides new mathematical theories to guide the design process.
In short, we obtain architectures by discretizing the underlying PDE through adequate time integration methods. 

In addition to substantial training costs, deep CNNs face fundamental challenges when it comes to their interpretability and robustness. 
In particular, CNNs that are used in mission-critical tasks (such as driverless cars) face the challenge of being "explainable." Casting the learning task within nonlinear PDE theory allows us to understand the properties of such networks better. We believe that further research into the mathematical structures presented here will result in a more solid understanding of the networks and will close the gap between deep learning and more mature fields that rely on nonlinear PDEs such as fluid dynamics.    
A direct impact of our approach can be observed when studying,
e.g., adversarial examples. Recent works~\cite{PeiEtAl2017} indicate that the predictions obtained by deep networks can be very sensitive to perturbations of the input images.
These findings motivate us to favor networks that are stable, i.e., networks whose output are robust to small perturbations of the input features, similar to what PDE analysis suggests.

In this paper, we consider residual neural networks (ResNet)~\cite{he2016deep}, a very effective type of neural networks. 
We show that residual CNNs can be interpreted as a discretization of a space-time differential equation. 
We use this link for analyzing the stability of a network
and for motivating new network models that bear similarities with well-known PDEs.
 Using our framework, we present three new architectures. 
First, we introduce parabolic CNNs that restrict the forward propagation to dynamics that smooth image features and bear similarities with anisotropic filtering~\cite{PeronaMalik1990,Weickert2009,ChenPock2017}. 
Second, we propose hyperbolic CNNs that are inspired by Hamiltonian systems and finally, a 
third, second-order hyperbolic CNN. 
As to be expected, those networks have different properties.
For example,  hyperbolic CNNs approximately preserve the energy in the system, which sets them apart from parabolic networks that smooth the image data, reducing the energy. 
Computationally, the structure of a hyperbolic forward propagation can be exploited to alleviate the memory burden because hyperbolic dynamics can be made reversible on the continuous and discrete levels. 
The methods suggested here are closely related to reversible ResNets~\cite{GomezEtAl2017,Chang2017Reversible}.

The remainder of this paper is organized as follows. 
In Section~\ref{sec1}, we provide a brief introduction into residual networks and their relation to ordinary and, in the case of convolutional neural networks, partial differential equations. 
In Section~\ref{sec3}, we present three novel CNN architectures motivated by PDE theory. 
Based on our continuous interpretation we present regularization functionals that enforce the smoothness of the dynamical systems, in Section~\ref{sec5}.
In Section~\ref{sec6}, we present numerical results for image classification that indicate the competitiveness of our PDE-based architectures.
Finally, we highlight some directions for future research in Section~\ref{sec7}.

\begin{figure}[t]
    \begin{center}
        \includegraphics[width=.49\textwidth,trim=110 6 195 0,clip=true]{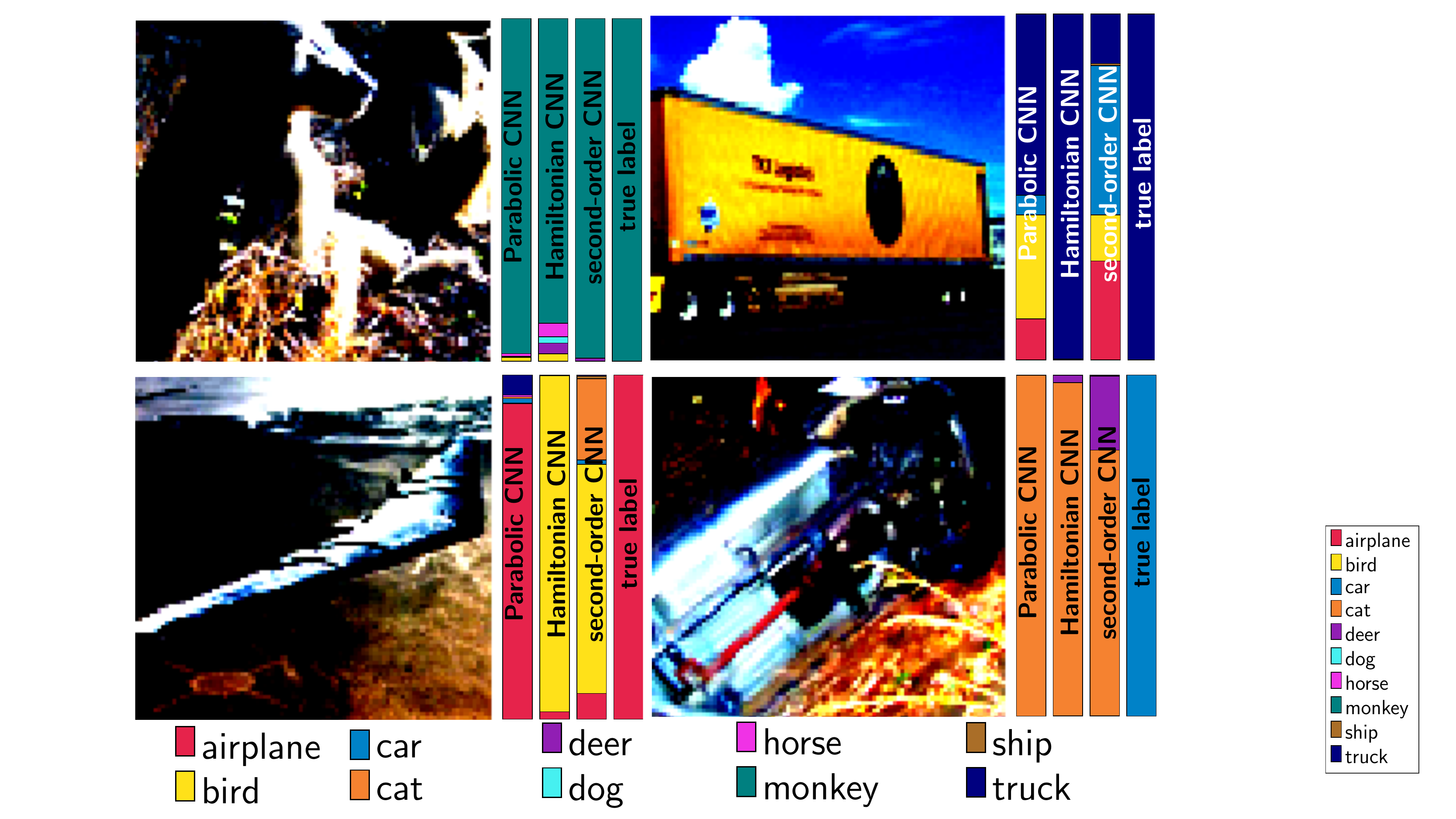}
    \end{center}
    \caption{Classification results of the three proposed CNN architecture for four randomly selected test images from the STL-10 dataset~\cite{CoatesEtAl2011}. 
    The predicted and actual class probabilities are visualized using bar plots on the right of each image.
    While all networks reach a competitive prediction accuracy between around 74\% and 78\% across the whole dataset, predictions for individual images vary in some cases. 	}\label{fig:CNNresults}    
\end{figure}

\section{Residual Networks and Differential Equations} 
\label{sec1}

The abstract goal of machine learning is to find a function $f : \R^n \times \R^p \to \R^m$ such that $f(\cdot, \bftheta) $ accurately predicts the result of an observed phenomenon (e.g., the class of an image, a spoken word, etc.).  The function is parameterized by the weight vector $\bftheta\in\R^p$ that is trained using examples.
In supervised learning, a set of input features $\bfy_1,\ldots,\bfy_s \in \R^n$ and output labels $\bfc_1, \ldots, \bfc_s \in \R^m$ is available and used to train the model $f(\cdot,\bftheta)$. 
The output labels are vectors whose components correspond to the estimated probability of a particular example belonging to a given class. 
As an example, consider the image classification results in Fig.~\ref{fig:CNNresults} where the predicted and actual labels are visualized using bar plots.
 For brevity, we denote the training data by  $\bfY = [\bfy_1, \bfy_2, \ldots, \bfy_s] \in \R^{n\times s}$ and $\bfC = [\bfc_1, \bfc_2, \ldots, \bfc_s] \in\R^{m\times s}$. 

In deep learning, the function $f$ consists of a concatenation of nonlinear functions called hidden layers. 
Each layer is composed of affine linear transformations and pointwise nonlinearities and aims at filtering the input features in a way that enables learning. 
As a fairly general formulation, we consider an extended version of the layer used in~\cite{he2016deep}, which filters the features $\bfY$ as follows
\begin{equation}\label{eq:unit}
    \bfF(\bftheta,\bfY) = \bfK_{2}(\bftheta^{(3)}) \sigma\left( \calN (\bfK_{1}(\bftheta^{(1)}) \bfY,\bftheta^{(2)}) \right).
\end{equation}
Here, the parameter vector, $\bftheta$, is partitioned into three parts where $\bftheta^{(1)}$ and $\bftheta^{(3)}$ parameterize the linear operators $\bfK_1(\cdot) \in \R^{k \times n}$ and $\bfK_2(\cdot) \in \R^{k_{\mathrm{out}} \times k}$, respectively, and $\bftheta^{(2)}$ are the parameters of the normalization layer $\calN$. 
The activation function $\sigma : \R \to \R$ is applied component-wise. 
Common examples are $\sigma(x) = \tanh(x) $ or the rectified linear unit (ReLU) defined as $\sigma(x) = \max(0,x)$. 
A deep neural network can be written by concatenating many of the layers given in~\eqref{eq:unit}.

When dealing with image data, it is common to group the features into different channels (e.g., for RGB image data there are three channels) and define the operators $\bfK_1$ and $\bfK_2$ as block matrices consisting of spatial convolutions. Typically each channel of the output image is computed as a weighted sum of each of the convolved input channels. To give an example, assume that $\bfK_1$ has three input and two output channels and denote by $\bfK_1^{(\cdot,\cdot)}(\cdot)$ a standard convolution operator \cite{HansenNagyOLeary2006}. In this case, we can write $\bfK_1$ as
\begin{equation}\label{eq:K}
    \bfK_1(\bftheta)  = \left( \begin{array}{rrr}
        \bfK_1^{(1,1)}(\bftheta^{(1,1)}) & \bfK_1^{(1,2)}(\bftheta^{(1,2)}) & \bfK_1^{(1,3)}(\bftheta^{(1,3)})\\
        \bfK_1^{(2,1)}(\bftheta^{(1,2)}) & \bfK_1^{(2,2)}(\bftheta^{(2,2)}) & \bfK_1^{(2,3)}(\bftheta^{(2,3)})\\
    \end{array}
    \right),
\end{equation}
where $\bftheta^{(i,j)}$ denotes the parameters of the stencil of the $(i,j)$-th convolution operator.

A common choice for $\calN$ in ~\eqref{eq:unit} is the \emph{batch normalization} layer~\cite{Ioffe:2015ud}. This layer computes the empirical mean and standard deviation of each channel in the input images across the spatial dimensions and examples and uses this information to normalize the statistics of the output images. 
While the coupling of different examples is counter-intuitive, its use is wide-spread and motivated by empirical evidence showing a faster convergence of training algorithms.
The weights $\bftheta^{(2)}$ represent scaling factors and biases (i.e., constant shifts applied to all pixels in the channel) for each output channel that are applied after the normalization.

ResNets have recently improved the state-of-the-art in several benchmarks including computer vision contests on image classification~\cite{hinton2012deep,KrizhevskySutskeverHinton2012,lecun2015deep}. Given the input features $\bfY_0 = \bfY$, a ResNet unit with $N$ layers produces a filtered version $\bfY_N$ as follows 
\begin{equation}\label{eq:ResNN}
    \bfY_{j+1} = \bfY_j + \bfF(\bftheta^{(j)},\bfY_j), \enskip \text{ for } \enskip j=0,1,\ldots,N-1,
\end{equation}
where $\bftheta^{(j)}$ are the weights (convolution stencils and biases) of the $j$th layer. To emphasize the dependency of this process on the weights, we denote $\bfY_N(\bftheta)$.

Note that the dimension of the feature vectors (i.e., the image resolution and the number of channels) is the same across all layers of a ResNets unit, which is limiting in many practical applications. 
Therefore, implementations of deep CNNs contain a concatenation of ResNet units with other layers that can change, e.g., the number of channels and the image resolution (see, e.g.,\cite{he2016deep,Chang2017Reversible}).  

In image recognition, the goal is to classify the output of~\eqref{eq:ResNN}, $\bfY_N(\bftheta)$, using, e.g., a
linear classifier modeled by a fully-connected layer, i.e., an affine transformation with a dense matrix.
To avoid confusion with the ResNet units we denote these transformations as $\bfW \bfY_N(\bftheta) + (\bfB_W \bfmu) \bfe_s^\top$, where the columns of $\bfB_W$ represent a distributed bias and $\bfe_s \in\R^s$ is a vector of all ones. 
The parameters of the network and the classifier are unknown and have to be learned. 
Thus, the goal of learning is to estimate the network parameters, $\bftheta$, and the weights of the classifier, $\bfW, \bfmu$, by approximately solving the optimization problem
\begin{align}\label{eq:opt}
    \min_{\bftheta,\bfW,\bfmu}\; &  \frac{1}{2} S(\bfW \bfY_N(\bftheta) + (\bfB_W\bfmu)\bfe_s^\top, \bfC) + R(\bftheta, \bfW, \bfmu),
\end{align}
where $S$ is a loss function, which is convex in its first argument, and $R$ is a convex regularizer discussed below.
Typical examples of loss functions are the least-squares function in regression and logistic regression or cross entropy functions in classification~\cite{GoodfellowEtAl2016}.

The optimization problem in~\eqref{eq:opt} is challenging for several reasons. 
First, it is a high-dimensional and non-convex optimization problem. Therefore one has to be content with local minima.
Second,  the computational cost per example is high, and the number of examples is large. 
Third, very deep architectures are prone to problems such as vanishing and exploding gradients~\cite{BengioEtAl1994} that may occur when the discrete forward propagation is unstable~\cite{HaberRuthotto2017}. 

\subsection{Residual Networks and ODEs} 
\label{sub:deep_networks_and_odes}
We derived a continuous interpretation of the filtering provided by ResNets in~\cite{HaberRuthotto2017}. Similar observations were made in~\cite{E2017,ChaudhariEtAl2017}. 
The ResNet in~\eqref{eq:ResNN} can be seen as a forward Euler discretization (with a fixed step size of $\delta_t=1$) of the initial value problem
\begin{equation}
    \begin{aligned}
    \partial_t\bfY(\bftheta,t) &= \bfF(\bftheta(t),\bfY(t)), \text{ for } t \in (0,T] \label{eq:ODE1}\\
    \bfY(\bftheta,0) &= \bfY_0.
    \end{aligned}
\end{equation}
Here, we introduce an artificial time $t \in [0,T]$. 
The depth of the network is related to the arbitrary final time $T$ and the magnitude of the matrices $\bfK_1$ and $\bfK_2$ in~\eqref{eq:unit}. 
This observation shows the relation between the learning problem~\eqref{eq:opt} and parameter estimation of a system of nonlinear ordinary differential equations. 
Note that this interpretation does not assume any particular structure of the layer $\bfF$. 

The continuous interpretation of ResNets can be exploited in several ways. One idea is to accelerate training by solving a hierarchy of optimization problems that gradually introduce new time discretization points for the weights, $\bftheta$~\cite{HaberHolthamRuthotto2017}.
Also, new numerical solvers based on optimal control theory have been proposed in~\cite{li2017maximum}. 
Another recent work~\cite{chen2018neural} uses more sophisticated time integrators to solve the forward propagation and the adjoint problem (in this context commonly called back-propagation), which is needed to compute derivatives of the objective function with respect to the network weights.

\subsection{Convolutional ResNets and PDEs} 
\label{sec2}
In the following, we consider learning tasks involving features given by speech, image, or video data. 
For these problems, the input features, $\bfY$, can be seen as a discretization of a continuous function $Y(x)$.
We assume that the matrices $\bfK_1 \in\R^{\tilde{w} \times w_{\mathrm{in}}}$ and  $\bfK_2 \in\R^{w_\mathrm{out} \times \tilde{w}}$ in~\eqref{eq:unit} represent convolution operators~\cite{HansenNagyOLeary2006}. 
The parameters $w_\mathrm{in}, \tilde{w}, $ and $w_{\mathrm{out}}$ denote the \emph{width} of the layer, i.e., they correspond to  the number of input, intermediate, and output features of this layer. 

We now show that a particular class of deep residual CNNs can be interpreted as nonlinear systems of PDEs. 
For ease of notation, we first consider a one-dimensional convolution of a feature with one channel and then outline how the result extends to higher space dimensions and multiple channels.

Assume that the vector $\bfy \in \R^n$ represents a one-dimensional grid function obtained by discretizing $y: [0,1]\to\R$ at the cell-centers of a regular grid with $n$ cells and a mesh size $h = 1/n$, i.e., for $i=1,2,\ldots,n$
\begin{equation*}
	\bfy = [y(x_1),\ldots,y(x_n)]^{\top} \quad \text{ with} \quad 	 x_i = \left(i-\hf\right) h.
\end{equation*} 
Assume, e.g., that the operator $\bfK_1 = \bfK_1(\bftheta) \in \R^{n\times n}$ in~\eqref{eq:unit} is parameterized by the stencil $\bftheta \in \R^3$. 
Applying a coordinate change, we see that
\begin{align*}
    \bfK_1(\bftheta) \bfy  & =  [\bftheta_{1}\, \bftheta_{2}\, \bftheta_{3}] * \bfy  \\
                       & = \left({\frac {\bfbeta_{1}}{4}} [1\, 2 \, 1] + {\frac {\bfbeta_{2}}{2h}}[-1 \, 0\, 1] + {\frac {\bfbeta_{3}}{h^{2}}} [-1\, 2 \,-1]\right) * \bfy.
\end{align*}
Here, the weights $\bfbeta\in\R^3$ are given by
\begin{equation*}
\label{abc}
\left(
\begin{array}{rrr}
    \frac 14  & -{\frac 1{2h}}  & -{\frac 1{h^{2}}} \\
    \hf       & 0                    & {\frac 2 {h^{2}}} \\
    \frac 14  & {\frac {1}{2h}}  & -{\frac 1{h^{2}}}
\end{array}
\right)
\begin{pmatrix}
\bfbeta_{1} \\ \bfbeta_{2} \\ \bfbeta_{3}
\end{pmatrix}
= \begin{pmatrix}
\bftheta_{1} \\ \bftheta_{2} \\ \bftheta_{3}
\end{pmatrix},
\end{equation*}
which is a non-singular linear system for any $h>0$. 
We denote by $\bfbeta(\bftheta)$ the unique solution of this linear system.
Upon taking the limit, $h \to 0$, this observation motivates one to parameterize the convolution operator as
\begin{equation*}
    \bfK_1(\bftheta) = \bfbeta_1(\bftheta) + \bfbeta_2(\bftheta) \partial_x + \bfbeta_3 (\bftheta)\partial_{x}^2.
\end{equation*}
The individual terms in the transformation matrix correspond to reaction, convection, diffusion and the bias term in~\eqref{eq:unit} is a source/sink term, respectively. 
Note that higher-order derivatives can be generated by multiplying different convolution operators or increasing the stencil size.

This simple observation exposes the dependence of learned weights on the image resolution, which can be exploited in practice, e.g., by multiscale training strategies~\cite{HaberHolthamRuthotto2017}. 
Here, the idea is to train a sequence of network using a coarse-to-fine hierarchy of image resolutions (often called image pyramid). 
Since both the number of operations and the memory required in training is proportional to the image size, this leads to immediate savings during training but also allows one to coarsen already trained networks to enable efficient evaluation.
In addition to computational benefits, ignoring fine-scale features when training on the coarse grid can also reduce the risk of being trapped in an undesirable local minimum, which is an observation also made in other image processing applications.

Our argument extends to higher spatial dimensions. 
In 2D, e.g., we can relate the $3\times 3$ stencil parametrized by $\bftheta\in\R^9$ to
\begin{equation*}
    \begin{split}
        \bfK_1(\bftheta)  =& \bfbeta_1(\bftheta) + \bfbeta_2(\bftheta) \partial_x + \bfbeta_3(\bftheta) \partial_y \\
                      &+ \bfbeta_4(\bftheta)\partial_x^2 + \bfbeta_5(\bftheta) \partial_y^2 + \bfbeta_6(\bftheta) \partial_{x} \partial_y \\ 
                     & + \bfbeta_7(\bftheta)\partial_x^2\partial_y + \bfbeta_8(\bftheta) \partial_x\partial_y^2 + \bfbeta_9(\bftheta) \partial_x^2 \partial_y^2.
    \end{split}
\end{equation*}
To obtain a fully continuous model for the layer in~\eqref{eq:unit}, we proceed the same way with $\bfK_2$. 
In view of\eqref{eq:K}, we note that when the number of input and output channels is larger than one, $\bfK_1$ and $\bfK_2$ lead to a system of coupled partial differential operators.

Given the continuous space-time interpretation of CNN we view the optimization problem~\eqref{eq:opt} as an optimal control problem and, similarly, see learning as a parameter estimation problem for the time-dependent nonlinear PDE~\eqref{eq:ODE1}.
Developing efficient numerical methods for solving PDE-constrained optimization problems arising in optimal control and parameter estimation has been a fruitful research endeavor and led to many advances in science and engineering (for recent overviews see, e.g.,~\cite{BorzSchulz2012,HerzogKunisch2010,BieglerEtAl2007}). 
Using the theoretical and algorithmic framework of optimal control in machine learning applications has gained some traction only recently (e.g., \cite{E2017,HaberRuthotto2017,Chang2017Reversible,li2017maximum,chen2018neural}).  


\section{Deep Neural Networks motivated by PDEs} 
\label{sec3}
It is well-known that not every time-dependent PDE is stable with respect to perturbations of the initial conditions~\cite{ascherBook}. 
Here, we say that the forward propagation in ~\eqref{eq:ODE1} is stable if  there is a constant $M>0$ {\em independent of $T$} such that
\begin{equation}
    \label{eq:stab}
            \|\bfY(\bftheta,T) - \tilde{\bfY}(\bftheta,T)\|_F \leq M \|\bfY(0) - \tilde{\bfY}(0)\|_F,
\end{equation}
where $\bfY$ and $\tilde{\bfY}$ are solutions of~\eqref{eq:ODE1} for different initial values and $\|\cdot\|_F$ is the Frobenius norm.
The stability of the forward propagation depends on the values of the weights $\bftheta$ that are chosen by solving~\eqref{eq:opt}. 
In the context of learning, the stability of the network is critical to provide robustness to small perturbations of the input images.
In addition to image noise, perturbations could also be added deliberately to mislead the network's prediction by an adversary. There is some recent evidence showing the existence of such perturbations that reliably mislead deep networks by being barely noticeable to a human observer (e.g.,~\cite{Goodfellow:2014tl,PeiEtAl2017,MoosaviDezfooli:ud}).

To ensure the stability of the network for all possible weights, we propose to restrict the space of CNNs. 
As examples of this general idea, we present three new types of residual CNNs that are motivated by parabolic and first- and second-order hyperbolic PDEs, respectively. 
The construction of our networks guarantees that the networks are stable forward and, for the hyperbolic network, stable backward in time.

Though it is common practice to model $\bfK_1$ and $\bfK_2$ in~\eqref{eq:unit} independently, we note that it is, in general, hard to show the stability of the resulting network. 
This is because, the Jacobian of $\bfF(\bftheta,\bfY)$ with respect to the features has the form
\begin{equation*}
    \bfJ_{\bfY} \bfF = \bfK_2(\bftheta) \ {\rm diag}(\sigma'(\bfK_1(\bftheta \bfY)))\ \bfK_1(\bftheta),
\end{equation*} 
where $\sigma'$ denotes the derivatives of the pointwise nonlinearity and for simplicity we assume $\calN(\bfY)=\bfY$. Even in this simplified setting, the spectral properties of $\bfJ_{\bfY}$, which impact the stability, are unknown for arbitrary choices of $\bfK_1$ and $\bfK_2$.
    
As one way to obtain a stable network, we introduce a symmetric version of the layer in~\eqref{eq:unit} by choosing $\bfK_2 = - \bfK_1^\top$ in~\eqref{eq:unit}. 
To simplify our notation, we drop the subscript of the operator and define the symmetric layer
\begin{equation}\label{eq:Fsym}
    \bfF_{\text{sym}}(\bftheta,\bfY) = -\bfK(\bftheta)^{\top} \sigma\left(\calN(\bfK(\bftheta) \bfY,\bftheta)\right).
\end{equation}

It is straightforward to verify that this choice leads to a negative semi-definite Jacobian for any non-decreasing activation function. As we see next, this choice also allows us to link the discrete network to different types of PDEs.

\subsection{Parabolic CNN} 
\label{sub:parabolic_cnn}
We define the parabolic CNN by using the symmetric layer from~\eqref{eq:Fsym} in the forward propagation, i.e., in the standard ResNet we replace the dynamic in~\eqref{eq:ODE1} by
\begin{equation}\label{eq:parabolic}
    \partial_t \bfY(\bftheta,t) = \bfF_{\mathrm{sym}}(\bftheta(t),\bfY(t)), \enskip \text{ for } t \in (0,T].
\end{equation}
Note that~\eqref{eq:parabolic} is equivalent to the heat equation if $\sigma(x) =x $, $\calN(\bfY)= \bfY$ and $\bfK(t) = \nabla$. This motivates us to refer to this network as a parabolic CNN. 
Nonlinear parabolic PDEs are widely used, e.g., to filter images~\cite{PeronaMalik1990,Weickert2009,ChenPock2017} and our interpretation implies that the networks can be viewed as an extension of such methods.

The similarity to the heat equation motivates us to introduce a new normalization layer motivated by total variation denoising.  For a single example $\bfy\in\R^n$ that can be grouped into $c$ channels, we define
\begin{equation}\label{eq:Ntv}
    \calN_{\rm tv}(\bfy) = {\rm diag}\left( \frac{1}{ \bfA^\top \sqrt{ \bfA (\bfy^2) + \epsilon}}  \right) \bfy,
\end{equation}
where the operator $\bfA \in \R^{n/c \times n}$ computes the sum over all $c$ channels for each pixel, the square, square root, and the division are defined component-wise, $0<\epsilon \ll 1$ is fixed.
 As for the batch norm layer, we implement $\calN_{\rm tv}$ with trainable weights corresponding to global scaling factors and biases for each channel. In the case that the convolution is reduced to a discrete gradient, $ \calN_{\rm tv}$ leads to the regular dynamics in TV denoising.

\paragraph{Stability.} 
\label{par:stability_}
Parabolic PDEs have a well-known decay property that renders them robust to perturbations of the initial conditions. 
For the parabolic CNN in~\eqref{eq:parabolic} we can show the following stability result.
\begin{theorem}\label{theo1}
    If the activation function $\sigma$ is monotonically non-decreasing, then the forward propagation through a parabolic CNN satisfies~\eqref{eq:stab}.
\end{theorem}
\begin{proof}
    For ease of notation, we assume that no normalization layer is used, i.e., $\calN(\bfY) = \bfY$ in \eqref{eq:parabolic}. 
    We then show that     $\bfF_{\mathrm{sym}}(\bftheta(t),\bfY)$
    is a monotone operator. 
    Note that for all $t \in [0,T]$
    \begin{eqnarray*}
         -\left( \sigma(\bfK(t) \bfY) - \sigma(\bfK(t) \bfY_\epsilon),  \bfK(t) (\bfY- \bfY_{\epsilon}) \right)
        \leq  0.
    \end{eqnarray*}
    Where $(\cdot,\cdot)$ is the standard inner product and the inequality follows from the monotonicity of the activation function, which shows that 
    \begin{equation*}
        \partial_t \|\bfY(t) - \bfY_{\epsilon}(t)\|_F^2 \leq 0.
    \end{equation*}
    Integrating this inequality over $[0,T]$ yields stability as in~\eqref{eq:stab}. 
    The proof extends straightforwardly to cases when a normalization layer with scaling and bias is included.
\end{proof}

One way to discretize the parabolic forward propagation~\eqref{eq:parabolic} is using the forward Euler method. 
Denoting the time step size by $\delta_t > 0$ this reads
\begin{equation*}
    \bfY_{j+1} = \bfY_{j} + \delta_t \bfF_{\mathrm{sym}}(\bftheta(t_j), \bfY_j), \enskip j=0,1,\ldots,N-1,
\end{equation*}
where $ t_j = j \delta_t$.
The discrete forward propagation of a given example $\bfy_0$ is stable if $\delta_t$ satisfies
\begin{equation*}
    \max_{i=1,2,\ldots,n} \left| 1+ \delta_t  \lambda_i(\bfJ(t_j))\right| \le 1, \enskip j=0,1,\ldots,N-1,
\end{equation*}
 and accurate if $\delta_t$ is chosen small enough to capture the dynamics of the system. 
Here, $\lambda_i(\bfJ(t_j))$ denotes the $i$th eigenvalue of the Jacobian of $\bfF_{\mathrm{sym}}$ with respect to the features at a time point $t_j$.
If we assume, for simplicity, that no normalization layer is used, the Jacobian is
\begin{align*}
    \bfJ(t_j) =  - \bfK^{\top}(&\bftheta^{(1)}(t_j))\  \bfD(t_j)  \bfK(\bftheta^{(1)}(t_j)),\\
    \text{ with}\quad&\bfD(t) = \diag\left( \sigma'\left(\bfK(\bftheta^{(1)}(t)) \bfy(t) \right)\ \right).
\end{align*}
If the activation function is monotonically nondecreasing, then $\sigma'(\cdot) \geq 0$  everywhere. 
In this case, all eigenvalues of $\bfJ(t_j)$ are real and bounded above by zero since $\bfJ(t_j)$ is also symmetric. 
Thus, there is an appropriate $\delta_t$ that renders the discrete forward propagation stable.
In our numerical experiments, we aim at ensuring the stability of the discrete forward propagation by limiting the magnitude of elements in $\bfK$ by adding bound constraints to the optimization problem~\eqref{eq:opt}.



\subsection{Hyperbolic CNNs} 
\label{sec4}
Different types of networks can be obtained by considering hyperbolic PDEs.
In this section, we present two CNN architectures that are inspired by hyperbolic systems. 
A favorable feature of hyperbolic equations is their reversibility. 
Reversibility allows us to avoid storage of intermediate network states, thus achieving higher memory efficiency. This is particularly important for very deep networks where memory limitation can hinder training (see \cite{GomezEtAl2017} and~\cite{Chang2017Reversible}).

\paragraph{Hamiltonian CNNs.} 
\label{par:hamiltonian_cnns}
Introducing an auxiliary variable $\bfZ$ (i.e., by partitioning the channels of the original features), we consider the dynamics
\begin{equation*}
\label{equ:two-layer-hamiltonian-simple}
\begin{aligned}
\partial_t{\mathbf{Y}}(t) &= - \bfF_{\text{sym}}(\bftheta^{(1)}(t),\bfZ(t)), \enskip &\bfY(0)=\bfY_0 \\ 
\partial_t{\mathbf{Z}}(t) &= \bfF_{\text{sym}}(\bftheta^{(2)}(t), \bfY(t)), \enskip &\bfZ(0)=\bfZ_0. 
\end{aligned}
\end{equation*}
We showed in~\cite{Chang2017Reversible} that the eigenvalues of the associated Jacobian are imaginary.
When assuming that $\bftheta^{(1)}$ and $\bftheta^{(2)}$ change sufficiently slow in time stability as defined in~\eqref{eq:stab} is obtained.  
A more precise stability result can be established by analyzing the kinematic eigenvalues of the forward propagation~\cite{amr}.

We discretize the dynamic using the symplectic Verlet integration (see, e.g., \cite{ascherBook} for details)
\begin{equation}
    \label{eqn:hamiltonian-discretized}
    \begin{aligned}
    \mathbf{Y}_{j+1} &= \mathbf{Y}_{j} + \delta_t \bfF_{\mathrm{sym}} (\bftheta^{(1)}(t),\bfZ_j), \\
    \mathbf{Z}_{j+1} &= \mathbf{Z}_{j} - \delta_t \bfF_{\mathrm{sym}}(\bftheta^{(2)}(t),\bfY_{j+1}),
    \end{aligned}
\end{equation} 
for $j=0,1,\ldots,N-1$ using a fixed step size $\delta_t > 0$.
This dynamic is \emph{reversible}, i.e., given $\bfY_{N}, \bfY_{N-1}$ and $\bfZ_N, \bfZ_{N-1}$ it can also be computed backwards
\begin{equation*}
    \begin{aligned}
        \mathbf{Z}_{j} &= \mathbf{Z}_{j+1} + \delta_t \bfF_{\mathrm{sym}}(\bftheta^{(2)}(t),\bfY_{j+1})\\
    \mathbf{Y}_{j} &= \mathbf{Y}_{j+1} -  \delta_t \bfF_{\mathrm{sym}}(\bftheta^{(1)}(t),\bfY_{j+1}),\enskip  \\
\end{aligned}
\end{equation*}
for $j=N-1,N-2,\ldots,0$. These operations are numerically stable for the Hamiltonian CNN (see~\cite{Chang2017Reversible} for details).

\paragraph{Second-order CNNs.} 
\label{par:second_order_cnns}
An alternative way to obtain hyperbolic CNNs is by using a second-order dynamics
\begin{equation} \label{eq:secondOrder}
    \begin{aligned}
    \partial_{t}^2 \bfY(t) &  = \bfF_{\text{sym}}(\bftheta(t),\bfY(t)), \\
    \bfY(0) &= \bfY_0,\enskip \partial_t \bfY(0) = 0.
    \end{aligned}
\end{equation}
The resulting forward propagation is associated with a nonlinear version of the telegraph equation \cite{rogers1984wave}, which describes the propagation of signals through networks. 
Hence, one could claim that second-order networks better mimic biological networks and are therefore more appropriate than first-order networks for approaches that aim at imitating the propagation through biological networks. 

We discretize the second-order network using the Leapfrog method. For $j=0,1,\ldots, N-1$ and $\delta_t>0$ fixed this reads
\begin{equation*}
    \bfY_{j+1} = 2\bfY_{j}-\bfY_{j-1} + \delta_t^2 \bfF_{\text{sym}}(\bftheta(t_j),\bfY_j).
\end{equation*}
We set $\bfY_{-1} = \bfY_0$ to denote the initial condition.
Similar to the symplectic integration in~\eqref{eqn:hamiltonian-discretized}, this scheme is reversible.

We show that the second-order network is stable in the sense of~\eqref{eq:stab} when we assume stationary weights.
Weaker results for the time-dependent dynamic are possible assuming $\partial_t \bftheta(t)$ to be bounded. 
\begin{theorem}
    Let $\bftheta(t)$ be constant in time and assume that the activation function satisfies $|\sigma(x)| \leq |x|$ for all $x$. Then, the forward propagation through the second-order network satisfies~\eqref{eq:stab}.
\end{theorem}
\begin{proof}
    For brevity, we denote $\bfK = \bfK(\bftheta(t))$ and consider the forward propagation of a single example. 
    Let $\bfy:[0,T]\to\R^n$ be a solution to~\eqref{eq:secondOrder} and consider the energy
    \begin{eqnarray}
    \label{energy}
    {\cal E}(t) =  \hf \left( (\partial_t\bfy(t))^{\top} \partial_t\bfy(t) + (\bfK \bfy(t))^{\top} \sigma(\bfK \bfy(t))  \right).
    \end{eqnarray}
    Given that $|\sigma(x)| \leq |x|$ for all $x$ by assumption, this energy can be bounded as follows
    \begin{align*}
    {\cal E}(t) & \leq {\cal E}_{\rm lin}(t)\\
                & = \hf \left( (\partial_t \bfu(t))^{\top} \partial_t\bfu(t) + (\bfK \bfu(t))^{\top} (\bfK \bfu(t))  \right),
    \end{align*}
    where ${\cal E}_{\rm lin}$ is the energy associated with the linear wave-like hyperbolic equation
    $$ \partial_t^2 \bfu(t) = -\bfK^{\top}\bfK \bfu(t), \enskip \bfu(0)=\bfy_0, \enskip \partial_t \bfu(0)=0. $$
    Since by assumption $\bfK$ is constant in time, we have that
    \begin{align*}
     \partial_t{\cal E}_{\rm lin}(t) =  \partial_t \bfu(t)^\top \left( \partial_t^2 \bfu(t) +  \bfK^{\top}\bfK \bfu(t) \right) 
      = 0. 
    \end{align*}
    Thus, the energy of the hyperbolic network in~\eqref{energy} is positive and bounded from above
    by the energy of the linear wave equation. 
    Applying this argument to the initial condition $\bfy_0 - \bfy_{\epsilon}$ we derive~\eqref{eq:stab} and thus the forward propagation is stable.
\end{proof}


\section{Regularization } 
\label{sec5}
The proposed continuous interpretation of the CNNs also provides new perspectives on regularization. 
To enforce stability of the forward propagation, the linear operator $\bfK$ in~\eqref{eq:Fsym} should not change drastically in time. 
This suggests adding a smoothness regularizer in time.
In~\cite{HaberRuthotto2017} a $H^1$-seminorm was used to smooth kernels over time with the goal to avoid overfitting.
A theoretically more appropriate function space consists of all kernels that are piecewise smooth in time. 
To this end, we introduce the regularizer
\begin{equation}\label{eq:TV}
    \begin{split}
    R(\bftheta,\bfW,\bfmu) = &  \alpha_1 \int_0^T \phi_\tau( \partial_t \bftheta(t)) dt \\
                          + & \frac{\alpha_2}{2} \left( \int_0^T \|\bftheta(t)\|^2 dt+ \| \bfW\|_F^2 + \|\bfmu\|^2\right),
    \end{split}
\end{equation}
where the function $\phi_{\epsilon}(x) = \sqrt{x^2 + \tau}$ is a smoothed $\ell_1$-norm with conditioning parameter $\tau>0$. The first term of $R$ can be seen as a total variation~\cite{RudinOsherFatemi1992} penalty in time that favors piecewise constant dynamics. Here, $\alpha_1, \alpha_2\geq 0$ are regularization parameters that are assumed to be fixed.

A second important aspect of stability is to keep the time step sufficiently small.
Since $\delta_t$ can be absorbed in $\bfK$ we use the box constraint
$-1\le \bftheta^{(1)}(t_j) \le 1$ for all $j$, and fix the time step size to $\delta_t=1$ in our numerical experiments.


\section{Numerical Experiments} 
\label{sec6}
%

We demonstrate the potential of the proposed architectures using the common image classification benchmarks STL-10~\cite{CoatesEtAl2011},  CIFAR-10, and CIFAR-100~\cite{krizhevsky2009learning}.
Our central goal is to show that, despite their modeling restrictions, our new network types achieve competitive results.
We use our basic architecture for all experiments, do not excessively tune hyperparameters individually for each case, and employ a simple data augmentation technique consisting of random flipping and cropping.

\paragraph{Network Architecture.} Our architecture is similar to the ones in~\cite{he2016deep,Chang2017Reversible} and contains an opening layer, followed by several blocks each containing a few time steps of a ResNet and a connector that increases the width of the CNN and coarsens the images.
Our focus is on the different options for defining the ResNet block using parabolic and hyperbolic networks.
To this end, we choose the same basic components for the opening and connecting layers.
The opening layer increases the number of channels from three (for RGB image data) to the number of channels of the first ResNet using convolution operators with $3\times3$ stencils, a batch normalization layer and a ReLU activation function. 
We build the connecting layers using $1\times1$ convolution operators that increase the number of channels, a batch normalization layer, a ReLU activation, and an average pooling operator that coarsens the images by a factor of two.
Finally, we obtain the output features $\bfY(\bftheta)$ by averaging the features of each channel to ensure translation-invariance. 
The ResNet blocks use the symmetric layer~\eqref{eq:Fsym} including the total variation normalization~\eqref{eq:Ntv} with $\epsilon=10^{-3}$.
The classifier is modeled using a fully-connected layer, a softmax transformation, and a cross-entropy loss. 

\paragraph{Training Algorithm. }
In order to estimate the weights, we use a standard stochastic gradient descent (SGD) method with momentum of $0.9$. We use a piecewise constant step size (in this context also called learning rate), starting with 0.1, which is decreased by a factor of 0.2 at a-priori specified epochs. 
For STL-10 and CIFAR-10 examples, we perform 60, 20, and 20 epochs with step sizes of $0.1, 0.02, 0.004$, respectively. 
For the more challenging CIFAR-100 data set, we use 60, 40, 40, 40, and 20 epochs with step sizes of 0.1, 0.02, 0.004, 0.0008, 0.00016, respectively.
In all examples, the SGD steps are computed using mini-batches consisting of 125 randomly chosen examples. 
For data augmentation, we apply a random horizontal flip ($50\%$ probability), pad the images by a factor of 1/16 with zeros into all directions and randomly crop the image by 1/8 of the pixels, counting from the lower-left corner. 
The training is performed using the open-source software Meganet
on a workstation running Ubuntu 16.04 and MATLAB 2018b with two Intel(R) Xeon(R) CPU E5-2620 v4 and 64 GB of RAM. 
We use  NVIDIA Titan X GPU for accelerating the computation through the frameworks CUDA 9.1 and CuDNN 7.0.

\paragraph{Results for STL-10. }
The STL-10 dataset~\cite{CoatesEtAl2011} contains 13,000 digital color images of size $96\times96$ that are evenly divided into ten categories, which can be inferred from Fig.~\ref{fig:CNNresults}.
The dataset is split into 5,000 training and 8,000 test images. 
The STL-10 data is a popular benchmark test for image classification algorithms and challenging due to the relatively small number of training images. 

\begin{figure}[t]
    \begin{center}
        \includegraphics[width=.30\textwidth]{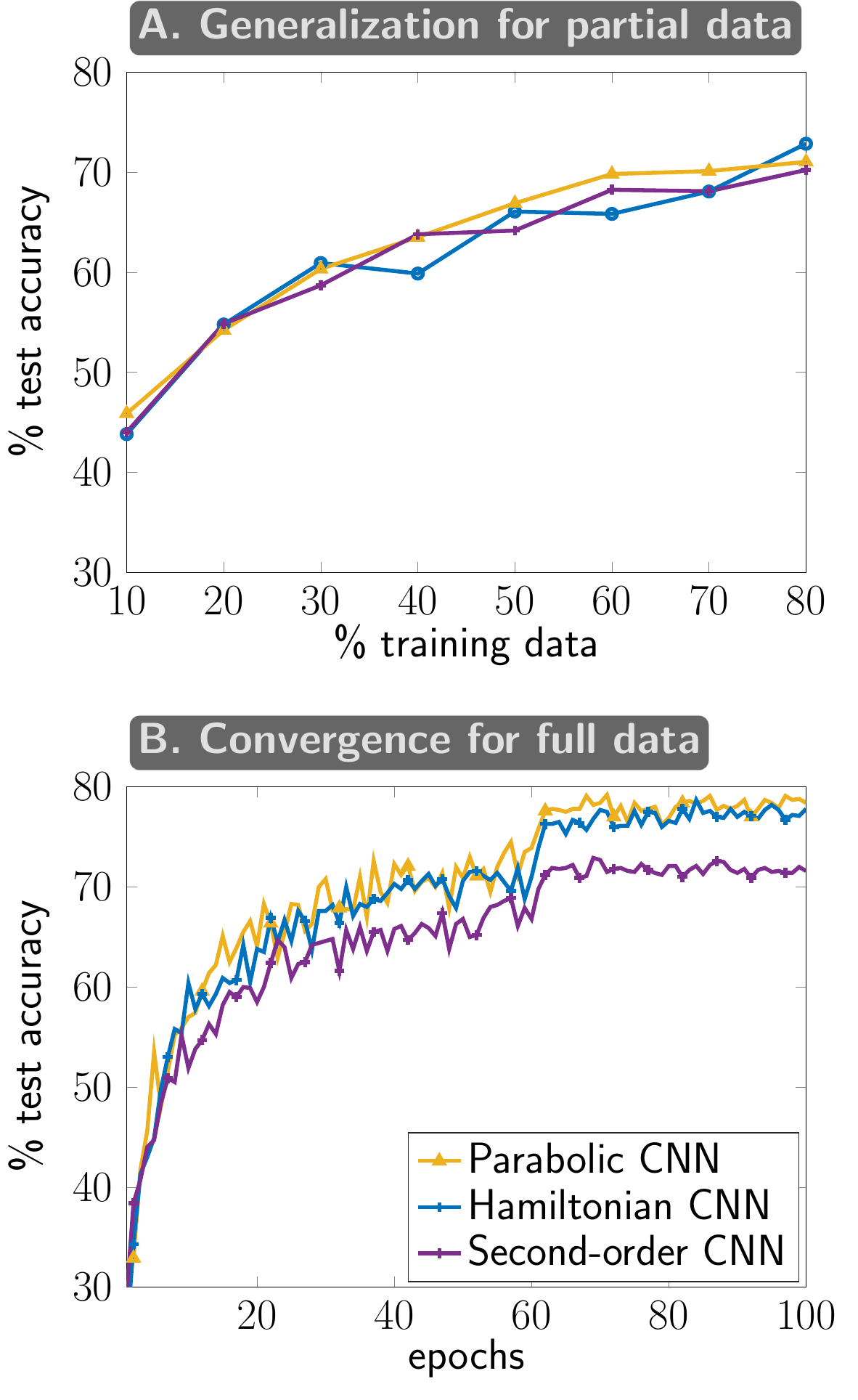}
    \end{center}
    \caption{Performance of the three proposed architectures for the STL-10 dataset. Top: Improvement of test accuracy when increasing the number of training images (10\% to 80\% in increments of 10\%). Bottom: Validation accuracy on remaining 20\% of training examples at every epoch of the stochastic gradient descent method. In this example, the parabolic and first-order hyperbolic architectures outperform the second-order network.}
    \label{fig:results}
\end{figure}
\begin{figure*}
    \begin{center}
    \includegraphics[width=1\textwidth]{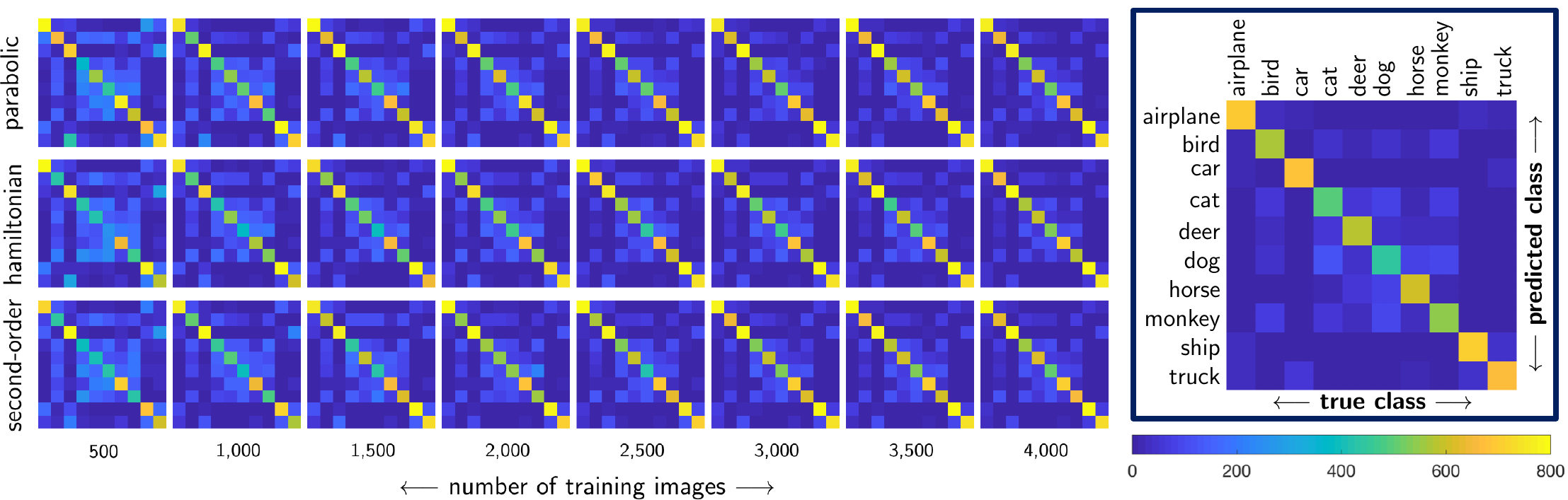}
    \end{center}
    \caption{Confusion matrices for classifiers obtained using the three proposed architectures (row-wise) for an increasing number of training data from the STL-10 dataset (column-wise). The $(i,j)$th element of the $10\times 10$ confusion matrix counts the number of images of class $i$ for which the predicted class is $j$. We use the entire test data set, which contains 800 images per class.}
    \label{fig:confMat}
\end{figure*}

For each dynamic, the network uses four ResNet blocks with 16, 32, 64, and 128 channels and image sizes of $96\times 96$, $48\times48$, $24\times 24$, $12\times 12$, respectively.
Within the ResNet blocks, we perform three time steps with a step size of $\delta_t=1$ and include a total variation normalization layer and ReLU activation. 
This architecture leads to 324,794 trainable weights for the Hamiltonian network and 618,554 weights for the parabolic and second-order network, respectively.
 We note that our network are substantially smaller than commonly used ResNets. For example, the architectures in~\cite{Chang2017Reversible} contain about 2 million parameters. Reducing the number of parameters is important during training and, e.g., when trained networks have to be deployed on devices with limited memory.
The regularization parameters are $\alpha_1 = 4\cdot 10^{-4}$ and $\alpha_2=1\cdot 10^{-4}$.

To show how the generalization improves as more training data becomes available, we train the network with an increasing number of examples that we choose randomly from the training dataset.
We also randomly sample 1,000 examples from the remaining training data to build a validation set, which we use to monitor the performance after each full epoch.
We use no data augmentation in this experiment.
 In all cases, the training accuracy was close to 100\%. After the training, we compute the accuracy of the networks parameterized by the weights that performed best on the validation data for all the 8,000 test images;  see Fig.~\ref{fig:results}.  
The predictions of the three networks may vary for single examples without any apparent pattern (see also Fig.~\ref{fig:CNNresults}).
However, overall their performance and convergence are comparable which leads to similarities in the confusion matrices; see Fig.~\ref{fig:confMat}.

To show the overall performance of the networks, we train the networks using a random partition of the examples into 4,000 training and 1,000 validation data. 
For data augmentation, we use horizontal flipping and random cropping. 
The performance of the networks on the validation data after each epoch can be seen in the bottom plot of Fig.~\ref{fig:results}.  
As before, the optimization found weights that almost perfectly fit the training data.
After the training, we compute the loss and classification accuracy for all the test images. 
For this example, the parabolic and Hamiltonian network perform slightly superior to the second-order network $77.0$\% and  $78.3$\% vs.  $74.3$\% test classification accuracy, respectively. 
It is important to emphasize that the Hamiltonian network achieves the best test accuracy using only about half as many trainable weights as the other two networks.
These results are competitive with the results reported, e.g., in~\cite{ShuoYang:2015vr,Dundar:2015ut}. Fine-tuning of hyperparameters such as step size, number of time steps, and width of the network may achieve additional improvements for each dynamic.  Using these means and performing training on all 5,000 images we achieved a test accuracy of around 85\% in~\cite{Chang2017Reversible}. 

\begin{table*}[t]
	\centering\scriptsize
    \caption{Summary of numerical results for the STL-10, CIFAR-10, and CIFAR-100 datasets. In each experiment, we randomly split the training data into 80\% used to train the weights and 20\% used to validate the performance after each epoch. After training, we compute and report the classification accuracy and the value of cross entropy loss (in brackets) for the test data. We evaluate the performance using the weights with the best classification accuracy on the validation set. We also report the number of trainable weights for each network.}
    \begin{center}
		\setlength{\tabcolsep}{4pt}		
    \begin{tabular}{|c|c|c|c|c|c|c|}
        \hline
        & \multicolumn{2}{|c|}{STL-10}
        & \multicolumn{2}{|c|}{CIFAR-10}
        & \multicolumn{2}{|c|}{CIFAR-100}\\ \cline{2-7}
                        &  number of      & test data (8,000)    & number of       & test data (10,000)   &  number of      & test data (10,000)\\
                        &  weights        & accuracy \%(loss)    &    weights      & accuracy \%(loss)    &    weights      & accuracy \%(loss) \\ \hline\hline
        Parabolic       & 618,554         & 77.0\% (0.711)       & 502,570         & 88.5\% (0.333)       & 652,484         & 64.8\% (1.234)    \\
        Hamiltonian     & 324,794         & 78.3\% (0.789)       & 264,106         & 89.3\% (0.349)       & 362,180         & 64.9\% (1.237)    \\
        Second-order    & 618,554         & 74.3\% (0.810)       & 502,570         & 89.2\% (0.333)       & 652,484         & 65.4\% (1.232)    \\ \hline
    \end{tabular}
    \end{center}
    \label{tab:res}
\end{table*}

\paragraph{Results for CIFAR 10/100. } For an additional comparison of the proposed architectures we use the CIFAR-10 and CIFAR-100 datasets~\cite{krizhevsky2009learning}. Each of these datasets consists of 60,000 labeled RGB images of size $32\times 32$ that are chosen from the 80 million tiny images dataset. In both cases, we use 50,000 images for training and validation and keep the remaining 10,000 to test the generalization of the trained weights. While CIFAR-10 consists of 10 categories the CIFAR-100 dataset contains 100 categories and, thus, classification is more challenging.

Our architectures contain three blocks of parabolic or hyperbolic networks between which the image size is reduced from $32\times 32$ to $8\times 8$. 
For the simpler CIFAR-10 problem, we use a narrower network with $32, 64, 112$ channels while for the CIFAR-100 challenge we use more channels (32, 64, and 128) and add a final connecting layer that increases the number of channels to 256. 
This leads to networks whose number of trainable weights vary between 264,106 and 652,484; see also Table~\ref{tab:res}. 
As regularization parameters we use  $\alpha_1 = 2\cdot 10^{-4}$ and $\alpha_2=2\cdot 10^{-4}$, which is similar to~\cite{Chang2017Reversible}.

As for the STL-10 data set, the three proposed architectures achieved comparable results on these benchmarks; see convergence plots in Figure~\ref{fig:cifar} and test accuracies in Table~\ref{tab:res}. In all cases, the training loss is near zero after the training. 
For these datasets, the second-order network slightly outperforms the other networks.
Additional tuning of the learning rate, regularization parameter, and other hyperparameters may further improve the results shown here. Using those techniques architectures with more time-steps and the entire training dataset we achieved about 5\% higher accuracy on CIFAR-10 and 9\% higher accuracy on CIFAR-100 in~\cite{Chang2017Reversible}.
\begin{figure}[t]
    \begin{center}
        \includegraphics[width=.3\textwidth]{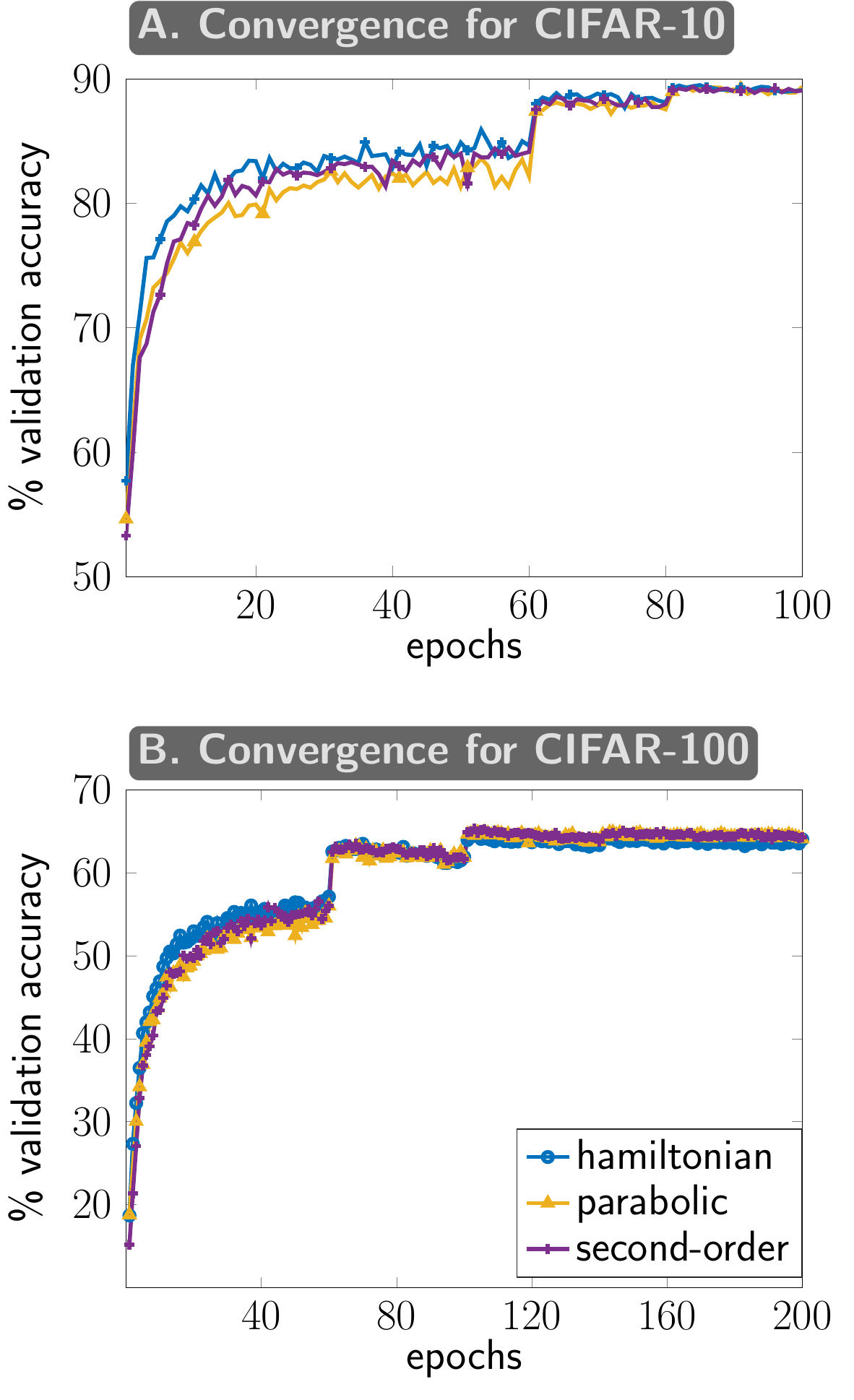}
    \end{center}
    \caption{Performance of the three proposed architectures for the CIFAR-10 (top) and CIFAR-100 (bottom) datasets. Validation accuracy computed on 10,000 randomly chosen images is shown at every epoch of the stochastic gradient descent method. In this example, all architectures perform comparably with the second-order network slightly outperforming the parabolic and first-order hyperbolic architectures.}
    \label{fig:cifar}
\end{figure}

\section{Discussion and Outlook} 
\label{sec7}
In this paper, we establish a link between deep residual convolutional neural networks and PDEs.  
The relation provides a general framework for designing, analyzing, and training those CNNs.
It also exposes the dependence of learned weights on the image resolution used in training.
Exemplarily, we derive three PDE-based network architectures that are forward stable (the parabolic network) and forward-backward stable (the hyperbolic networks). 

It is well-known that different types of PDEs have different properties. 
For example, linear parabolic PDEs have decay properties while linear hyperbolic PDEs conserve energy. 
Hence, it is common to choose different numerical techniques for solving and optimizing different kinds of PDEs.
The type of the underlying PDE is not known a-priori for a standard convolutional ResNet as it depends on the trained weights.
This renders ensuring the stability of the trained network and the choice of adequate time-integration methods difficult. 
These considerations motivate us to restrict the convolutional ResNet architecture a-priori to discretizations of nonlinear PDEs that are stable. 

In our numerical examples, our new architectures lead to an adequate performance despite the constraints on the networks.
In fact, using only networks of relatively modest size, we obtain results that are close to those of state-of-the-art networks with a considerably larger number of weights.
This may not hold in general, and future research will show which types of architectures are best suited for a learning task at hand. 
Our intuition is that, e.g., hyperbolic networks may be preferable over parabolic ones for image extrapolation tasks to ensure the preservation of edge information in the images. 
In contrast to that, we anticipate parabolic networks to perform superior for tasks that require filtering, e.g., image denoising.

We note that our view of CNNs mirrors the developments in PDE-based image processing in the 1990s. PDE-based methods have since significantly enhanced our mathematical understanding of image processing tasks and opened the door to many popular algorithms and techniques. 
We hope that continuous models of CNNs will result in similar breakthroughs and, e.g., help streamline the design of network architectures and improve training outcomes with less trial and error. 

\section*{Acknowledgements} 
	L.R. is supported by the U.S. National Science Foundation (NSF) through awards DMS 1522599 and DMS 1751636 and by the NVIDIA Corporation's GPU grant program. We thank Martin Burger for outlining how to show stability using monotone operator theory and Eran Treister and other contributors of the Meganet package. We also thank the Isaac Newton Institute (INI) for Mathematical Sciences for support and hospitality during the programme on Generative Models, Parameter Learning and Sparsity (VMVW02) when work on this paper was undertaken. INI was supported by EPSRC Grant Number: LNAG/036, RG91310. 

\bibliographystyle{abbrv}
\scriptsize

\begin{thebibliography}{10}

\bibitem{AmbrosioTortorelli1990}
L.~Ambrosio and V.~M. Tortorelli.
\newblock {Approximation of Functionals Depending on Jumps by Elliptic
  Functionals via Gamma-Convergence}.
\newblock {\em Commun. Pure Appl. Math.}, 43(8):999--1036, 1990.

\bibitem{ascherBook}
U.~Ascher.
\newblock {\em Numerical methods for Evolutionary Differential Equations}.
\newblock SIAM, Philadelphia, USA, 2010.

\bibitem{amr}
U.~Ascher, R.~Mattheij, and R.~Russell.
\newblock {\em Numerical Solution of Boundary Value Problems for Ordinary
  Differential Equations}.
\newblock SIAM, Philadelphia, Philadelphia, 1995.

\bibitem{bengio2009learning}
Y.~Bengio et~al.
\newblock Learning deep architectures for {AI}.
\newblock {\em Found. Trends Mach. Learn.}, 2(1):1--127, 2009.

\bibitem{BengioEtAl1994}
Y.~Bengio, P.~Simard, and P.~Frasconi.
\newblock {Learning Long-Term Dependencies with Gradient Descent Is Difficult}.
\newblock {\em IEEE Transactions on Neural Networks}, 5(2):157--166, 1994.

\bibitem{BieglerEtAl2007}
L.~T. Biegler, O.~Ghattas, M.~Heinkenschloss, D.~Keyes, and B.~van
  Bloemen~Waanders, editors.
\newblock {\em {Real-time PDE-constrained Optimization}}.
\newblock Society for Industrial and Applied Mathematics (SIAM), 2007.

\bibitem{BorzSchulz2012}
A.~Borz{\`\i} and V.~Schulz.
\newblock {\em {Computational optimization of systems governed by partial
  differential equations}}, volume~8.
\newblock Society for Industrial and Applied Mathematics (SIAM), Philadelphia,
  PA, 2012.

\bibitem{ChanVese1999}
T.~F. Chan and L.~A. Vese.
\newblock {Active contours without edges}.
\newblock {\em IEEE Trans. Image Process.}, 10(2):266--277, 2001.

\bibitem{Chang2017Reversible}
B.~Chang, L.~Meng, E.~Haber, L.~Ruthotto, D.~Begert, and E.~Holtham.
\newblock Reversible architectures for arbitrarily deep residual neural
  networks.
\newblock In {\em AAAI Conference on AI}, 2018.

\bibitem{ChaudhariEtAl2017}
P.~Chaudhari, A.~Oberman, S.~Osher, S.~Soatto, and G.~Carlier.
\newblock {Deep Relaxation: Partial Differential Equations for Optimizing Deep
  Neural Networks}.
\newblock {\em arXiv preprint 1704.04932}, Apr. 2017.

\bibitem{chen2018neural}
T.~Q. Chen, Y.~Rubanova, J.~Bettencourt, and D.~Duvenaud.
\newblock Neural ordinary differential equations.
\newblock {\em arXiv preprint arXiv:1806.07366}, 2018.

\bibitem{ChenPock2017}
Y.~Chen and T.~Pock.
\newblock {Trainable Nonlinear Reaction Diffusion: A Flexible Framework for
  Fast and Effective Image Restoration.}
\newblock {\em IEEE Trans. Pattern Anal. Mach. Intell.}, 39(6):1256--1272,
  2017.

\bibitem{CoatesEtAl2011}
A.~Coates, A.~Ng, and H.~Lee.
\newblock {An Analysis of Single-Layer Networks in Unsupervised Feature
  Learning}.
\newblock In {\em Proceedings of the 14th International Conference on
  Artificial Intelligence and Statistics}, pages 215--223, June 2011.

\bibitem{Dundar:2015ut}
A.~Dundar, J.~Jin, and E.~Culurciello.
\newblock {Convolutional Clustering for Unsupervised Learning}.
\newblock In {\em ICLR}, Nov. 2015.

\bibitem{E2017}
W.~E.
\newblock {A Proposal on Machine Learning via Dynamical Systems}.
\newblock {\em Comm. Math. Statist.}, 5(1):1--11, 2017.

\bibitem{GomezEtAl2017}
A.~N. Gomez, M.~Ren, R.~Urtasun, and R.~B. Grosse.
\newblock The reversible residual network: Backpropagation without storing
  activations.
\newblock In {\em Adv Neural Inf Process Syst}, pages 2211--2221, 2017.

\bibitem{GoodfellowEtAl2016}
I.~Goodfellow, Y.~Bengio, and A.~Courville.
\newblock {\em {Deep Learning}}.
\newblock MIT Press, Nov. 2016.

\bibitem{Goodfellow:2014tl}
I.~J. Goodfellow, J.~Shlens, and C.~Szegedy.
\newblock {Explaining and Harnessing Adversarial Examples}.
\newblock {\em arXiv.org}, Dec. 2014.

\bibitem{HaberRuthotto2017}
E.~Haber and L.~Ruthotto.
\newblock Stable architectures for deep neural networks.
\newblock {\em Inverse Probl.}, 34:014004, 2017.

\bibitem{HaberHolthamRuthotto2017}
E.~Haber, L.~Ruthotto, and E.~Holtham.
\newblock Learning across scales - {A} multiscale method for convolution neural
  networks.
\newblock In {\em AAAI Conference on AI}, volume abs/1703.02009, pages 1--8,
  2017.

\bibitem{HansenNagyOLeary2006}
P.~C. Hansen, J.~G. Nagy, and D.~P. O'Leary.
\newblock {\em {Deblurring Images: Matrices, Spectra and Filtering}}.
\newblock Matrices, Spectra, and Filtering. SIAM, Philadelphia, USA, 2006.

\bibitem{he2016deep}
K.~He, X.~Zhang, S.~Ren, and J.~Sun.
\newblock Deep residual learning for image recognition.
\newblock In {\em Proceedings of the IEEE Conference on Computer Vision and
  Pattern Recognition}, pages 770--778, 2016.

\bibitem{hernandez2016general}
J.~M. Hern{\'a}ndez-Lobato, M.~A. Gelbart, R.~P. Adams, M.~W. Hoffman, and
  Z.~Ghahramani.
\newblock A general framework for constrained bayesian optimization using
  information-based search.
\newblock {\em J. Mach. Learn. Res.}, 17:2--51, 2016.

\bibitem{HerzogKunisch2010}
R.~Herzog and K.~Kunisch.
\newblock {Algorithms for PDE-constrained optimization}.
\newblock {\em GAMM-Mitteilungen}, 33(2):163--176, Oct. 2010.

\bibitem{hinton2012deep}
G.~Hinton, L.~Deng, D.~Yu, G.~E. Dahl, A.-r. Mohamed, N.~Jaitly, A.~Senior,
  V.~Vanhoucke, P.~Nguyen, T.~N. Sainath, et~al.
\newblock Deep neural networks for acoustic modeling in speech recognition: The
  shared views of four research groups.
\newblock {\em IEEE Signal Process. Mag.}, 29(6):82--97, 2012.

\bibitem{HornSchunck1981}
B.~K. Horn and B.~G. Schunck.
\newblock Determining optical flow.
\newblock {\em Artificial intelligence}, 17(1-3):185--203, 1981.

\bibitem{Ioffe:2015ud}
S.~Ioffe and C.~Szegedy.
\newblock {Batch Normalization: Accelerating Deep Network Training by Reducing
  Internal Covariate Shift}.
\newblock In {\em 32nd International Conference on Machine Learning}, pages
  448--456, Feb. 2015.

\bibitem{krizhevsky2009learning}
A.~Krizhevsky and G.~Hinton.
\newblock Learning multiple layers of features from tiny images.
\newblock 2009.

\bibitem{KrizhevskySutskeverHinton2012}
A.~Krizhevsky, I.~Sutskever, and G.~Hinton.
\newblock Imagenet classification with deep convolutional neural networks.
\newblock {\em Adv Neural Inf Process Syst}, 61:1097–1105, 2012.

\bibitem{LeCunBengio1995}
Y.~LeCun and Y.~Bengio.
\newblock Convolutional networks for images, speech, and time series.
\newblock {\em The handbook of brain theory and neural networks},
  3361:255–258, 1995.

\bibitem{lecun2015deep}
Y.~LeCun, Y.~Bengio, and G.~Hinton.
\newblock Deep learning.
\newblock {\em Nature}, 521(7553):436--444, 2015.

\bibitem{LeCunKavukcuogluFarabet2010}
Y.~LeCun, K.~Kavukcuoglu, and C.~Farabet.
\newblock Convolutional networks and applications in vision.
\newblock {\em IEEE International Symposium on Circuits and Systems: Nano-Bio
  Circuit Fabrics and Systems}, page 253–256, 2010.

\bibitem{li2017maximum}
Q.~Li, L.~Chen, C.~Tai, and E.~Weinan.
\newblock Maximum principle based algorithms for deep learning.
\newblock {\em The Journal of Machine Learning Research}, 18(1):5998--6026,
  2017.

\bibitem{modersitzki2009fair}
J.~Modersitzki.
\newblock {\em FAIR: Flexible Algorithms for Image Registration}.
\newblock Fundamentals of Algorithms. SIAM, Philadelphia, USA, 2009.

\bibitem{MoosaviDezfooli:ud}
S.~M. Moosavi-Dezfooli, A.~Fawzi, O.~F. arXiv, and {2017}.
\newblock {Universal adversarial perturbations}.
\newblock {\em openaccess.thecvf.com}.

\bibitem{MumfordShah1989}
D.~Mumford and J.~Shah.
\newblock {Optimal Approximations by Piecewise Smooth Functions and Associated
  Variational-Problems}.
\newblock {\em Commun. Pure Appl. Math.}, 42(5):577--685, 1989.

\bibitem{PeiEtAl2017}
K.~Pei, Y.~Cao, J.~Yang, and S.~Jana.
\newblock Deepxplore: Automated whitebox testing of deep learning systems.
\newblock In {\em 26th Symposium on Oper. Sys. Princ.}, pages 1--18. ACM Press,
  New York, USA, 2017.

\bibitem{PeronaMalik1990}
P.~Perona and J.~Malik.
\newblock Scale-space and edge detection using anisotropic diffusion.
\newblock {\em IEEE Trans. Pattern Anal. Mach. Intell.}, 12(7):629--639, 1990.

\bibitem{RainaEtAl2009}
R.~Raina, A.~Madhavan, and A.~Y. Ng.
\newblock {Large-scale deep unsupervised learning using graphics processors}.
\newblock In {\em 26th Annual International Conference}, pages 873--880, New
  York, USA, 2009. ACM.

\bibitem{rogers1984wave}
C.~Rogers and T.~Moodie.
\newblock {\em Wave Phenomena: Modern Theory and Applications}.
\newblock North-Holland Mathematics Studies. Elsevier Science, 1984.

\bibitem{Rosenblatt1958}
F.~Rosenblatt.
\newblock {The perceptron: A probabilistic model for information storage and
  organization in the brain.}
\newblock {\em Psychological review}, 65(6):386--408, 1958.

\bibitem{RudinOsherFatemi1992}
L.~I. Rudin, S.~Osher, and E.~Fatemi.
\newblock {Nonlinear Total Variation Based Noise Removal Algorithms}.
\newblock {\em Physica D}, 60(1-4):259--268, 1992.

\bibitem{scherzer2009variational}
O.~Scherzer, M.~Grasmair, H.~Grossauer, M.~Haltmeier, and F.~Lenzen.
\newblock {\em Variational methods in imaging}.
\newblock Springer, New York, USA, 2009.

\bibitem{ShuoYang:2015vr}
P.~L. C. C. L. W. K. S. X.~T. Shuo~Yang.
\newblock {Deep Visual Representation Learning with Target Coding}.
\newblock In {\em AAAI Conference on AI}, pages 3848--3854, Jan. 2015.

\bibitem{Weickert2009}
J.~Weickert.
\newblock {\em {Anisotropic Diffusion in Image Processing}}.
\newblock 2009.

\end{thebibliography}

\end{document}